\documentclass{article}

\PassOptionsToPackage{numbers, compress}{natbib}

\usepackage[preprint]{neurips_2019}

\usepackage{natbib}
\usepackage{hyperref}
\usepackage{url}
\usepackage{graphicx}
\usepackage{subfigure}
\usepackage{booktabs}
\usepackage{makecell}
\usepackage{amssymb}
\usepackage{tikz}
\usepackage{soul}
\usepackage{hyperref}
\usepackage[font=small]{caption}
\usepackage{wrapfig}
\usepackage{enumitem}
\usepackage{url}
\usepackage[utf8]{inputenc} 
\usepackage[T1]{fontenc}    
\usepackage{amsfonts}       
\usepackage{nicefrac}       
\usepackage{amsmath}
\usepackage{cleveref}
\usepackage{multirow}
\usepackage{tabularx}
\usepackage{wrapfig}
\usepackage{lipsum}
\usepackage{array}
\usepackage{mathtools}
\usepackage{icomma}
\usepackage{siunitx}
\usepackage{xspace}
\usepackage[subtle,tracking=normal]{savetrees}

\usepackage{amsmath,amsthm,amssymb}

\usepackage{pst-node}
\usepackage{tikz-cd}

\newcommand\cut[1]{}

\newcolumntype{C}[1]{>{\centering\arraybackslash}m{#1}}
\newcolumntype{R}[1]{>{\raggedleft\arraybackslash}m{#1}}

\newcommand{\be}{\begin{equation}}
\newcommand{\ee}{\end{equation}}
\newcommand{\bea}{\begin{eqnarray}}
\newcommand{\eea}{\end{eqnarray}}
\newcommand{\beaa}{\begin{eqnarray*}}
\newcommand{\eeaa}{\end{eqnarray*}}

\DeclareMathAlphabet{\mathpzc}{OT1}{pzc}{m}{n}

\newboolean{draftmode}
\setboolean{draftmode}{true}

\usepackage{color}
\newcommand{\mycomment}[3]{{\textcolor{#3}{[#1 #2]}}}
\newcommand{\drmarker}{{\textcolor{red}{\ensuremath{^{\textsc{D}}_{\textsc{R}}}}}}
\newcommand{\gpmarker}{{\textcolor{purple}{\ensuremath{^{\textsc{G}}_{\textsc{P}}}}}}
\newcommand{\ihmarker}{{\textcolor{blue}{\ensuremath{^{\textsc{I}}_{\textsc{H}}}}}}

\ifthenelse{\boolean{draftmode}}{
 \newcommand{\dr}[1]{\mycomment{\drmarker}{#1}{red}}
 \newcommand{\gp}[1]{\mycomment{\gpmarker}{#1}{purple}} 
 \newcommand{\ih}[1]{\mycomment{\ihmarker}{#1}{blue}}

}{
 \newcommand{\dr}[1]{}
 \newcommand{\gp}[1]{}
 \newcommand{\ih}[1]{} 
}

\newtheorem{lemma}{Lemma}

\title{Equivariant Hamiltonian Flows}

\author{
Danilo J. Rezende* \And Sébastien Racani\`ere* \And Irina Higgins* \And Peter Toth*
\AND
*\{danilor, sracaniere, irinah, petertoth\}@google.com
}

\begin{document}

\maketitle

\begin{abstract}
This paper introduces equivariant hamiltonian flows, a method for learning expressive densities that are invariant with respect to a known Lie-algebra of local symmetry transformations while providing an equivariant representation of the data. We provide proof of principle demonstrations of how such flows can be learnt, as well as how the addition of symmetry invariance constraints can improve data efficiency and generalisation. Finally, we make connections to disentangled representation learning and show how this work relates to a recently proposed definition.
\end{abstract}

\section{Introduction}
Learning generative models with structured latent representations is important for interpretability, data efficiency, and generalisation \cite{bengio2013representation, gens2014symmetry}. One kind of structural bias that has been shown to work well in the past is that of invariance or equivariance with respect to a group of transformations. For example, convolutional neural networks \cite{lecun1989conv} are invariant to the group of translations by construction, while disentangled representations learn equivariance to more general transformations \cite{higgins2018towards}. In order to learn such structured latent representations, however, often a trade-off has to be made in terms of latent expressivity. For example, most of the disentangling models to date use a unit Gaussian prior \cite{higgins2017betavae, kim2018factorvae, chen2018isolating}.  
Normalizing flows \cite{tabak2010density, rezende2015variational, dinh2016density, kingma2016improved, NIPS2018_7892, papamakarios2017masked, kingma2018glow} provide a simple mechanism to build expressive density estimators in problems without much domain knowledge. But it remains a challenge to embed flow-based models with domain knowledge such as data on specific manifolds \cite{gemici2016normalizing} and known symmetries or factorisation of the target density. In this paper we focus on problems where there is domain knowledge in the form of known invariances of a target density which we wish to learn. This is the typical case in many modelling problems, such as density over molecular structures with various rotational SO$(n)$ symmetries, lattice-QCD \cite{albergo2019flow} with internal gauge symmetries (e.g. U$(n)$ and SU$(n)$), etc. The challenge we address is the following: suppose that we start from a base distribution $\pi(z)$ that is invariant with respect to a group of transformations. If we transform this density via a generic normalizing flow $f(z)$, there will be no guarantees that the transformed density $p(z)=\pi(f^{-1}(z))/\det|Jac|$ will be invariant as well. 

Our main contributions are: (i) Introduce equivariant Hamiltonian flows; (ii) Propose a general algorithm for enforcing equivariance in learned Hamiltonian flows with respect to any connected Lie-Group; and (iii) Prove a simple lemma that allow us to construct invariant densities from equivariant flows. The result of this are flows that transform a simple density that is invariant with respect to the actions of a known symmetry group, into another invariant density that is arbitrarily complex. We demonstrate that learning latent representations with the known equivariance structure helps with data efficiency and generalisation. Finally, we connect this work to the recent definition of disentangled representations \cite{higgins2018towards}.
\section{Methods}
We first revise the notion of Hamiltonian dynamics from physics \citep{goldstein1980mechanics}. In physics the Hamiltonian formalism is used to model energy conserving continuous dynamics in an abstract state space $s = (q, p)$ of generalised position $q$ and momentum $p$. 
Formal connections can be made between Hamiltonian dynamics and the actions of continuous symmetry generators -- those transformations that leave the dynamics of the system modelled by the Hamiltonian unchanged throughout the path traversed. Indeed, the roles of the Hamiltonian and its symmetry generators can be interchangeable, whereby a symmetry generator of one system can be a Hamiltonian for a different system, and both Hamiltonian and symmetry transformations commute with each other. The commutativity of transformations is measured by an operator known as the Poisson bracket. Hence, we use the Poisson bracket between the Hamiltonian and the known symmetry generators as a regularizer to restrict the form of the learnt Hamiltonian flow, whereby it models the final data density well, while also being invariant to the known symmetries.

\begin{figure*}[t]
\vskip -0.3in
\begin{center}
\centerline{\includegraphics[width=\linewidth]{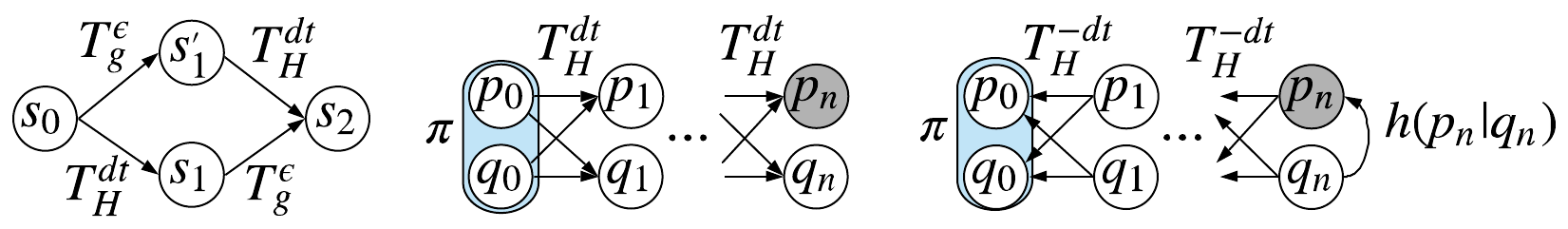}}
\caption{Model diagrams. From left to right: (i) Equivalence between flows with swapped $T_H^{dt}$ and $T_g^{\epsilon}$, as a result of the equivariance of Hamiltonian flows with respect to generators of symmetries of the Hamiltonian function; (ii) Generative model diagram. The last momentum variable $p_n$ (shaded node) is dropped during sampling; (iii) Inference mode diagram for computing the ELBO, where $h(p_n | q_n)$ is an encoder model. The last momentum variable $p_n$ is inferred during training.}
\label{fig.equivariance}
\end{center}
\vskip -0.3in
\end{figure*}

\paragraph{Hamiltonian Flow Generative Model}
\label{sec.hamiltonian.flows}
A Hamiltonian flow is a continuous-time normalizing flow induced by the Hamiltonian dynamics in the state-space $s = (q, p) \in \mathbb{R}^{2d}$ via the ODE $(\dot{q}, \dot{p}) = \{(q, p), ~H(q, p)\} = (\frac{\partial H}{\partial p}, -\frac{\partial H}{\partial q})$, where $\{,\}$ is a skew-symmetric differential operator known as the Poisson bracket or commutator and $H: \mathbb{R}^{2d} \rightarrow \mathbb{R}$ is a scalar function known as Hamiltonian. For two scalar functions $f$, $g$ we set $\{f(q, p),~g(q, p)\} = \sum_i \frac{\partial f(q, p)}{\partial q_i} \frac{\partial g(q, p)}{\partial p_i}
- \frac{\partial f(q, p)}{\partial p_i} \frac{\partial g(q, p)}{\partial q_i}$. If $f$ or $g$ is a vector valued function, we extend the previous definition component wise so that $\{f,~g\}$ is also a vector value function.

In what follows, we introduce the notation $T_f^{\epsilon}(x) = x + \epsilon \{x, f(x)\}$ to indicate an infinitesimal transformation induced by the Poisson bracket with a function $f$. With this notation, the Euler discretization of the Hamiltonian ODE is written as $s_{t+dt} = T_H^{dt}(s_t) = s_t + dt \{s_t, H(s_t)\}$. We can construct a more complex flow by chaining several transformations $s' = T_{H_n}^{dt} \circ \ldots \circ T_{H_1}^{dt}(s)$. The inverse of this flow is obtained by replacing $dt$ by $-dt$ and reversing the order of the transformations,
$s = T_{H_1}^{-dt} \circ \ldots \circ T_{H_n}^{-dt}(s') + O(dt^2)$.

We can create an expressive probability density $p_{\theta}(s_n)$ on $s_n=(q_n, p_n)$ by starting from an initial distribution $\pi(s_0)$, and transforming this density via the flow $s_n = T_{H_n}^{dt} \circ \ldots \circ T_{H_1}^{dt}(s_0)$, where the Hamiltonians $H_i$ are scalar functions parametrized by neural networks with parameters $\theta$. This constitutes a more structured form of Neural ODE flow, \cite{NIPS2018_7892}. There are a few notable differences between Hamiltonian flows and general ODE flows: (i) The Hamiltonian ODE is volume-preserving \footnote{
This can be seen by computing the determinant of the Jacobian of the infinitesimal transformation $T_H^{dt}$, 
$\text{det}(\text{Jac}) = \text{det} \left[
    \mathbb{I} + dt
    \left(\begin{array}{cc}
      \frac{\partial^2 H}{\partial q_i \partial p_j} & - \frac{\partial^2
      H}{\partial q_i \partial q_j}\\
      \frac{\partial^2 H}{\partial p_i \partial p_j} & - \frac{\partial^2
      H}{\partial p_i \partial q_j}
    \end{array}\right)
    \right] = 1 + dt~\text{Tr}\left(\begin{array}{cc}
      \frac{\partial^2 H}{\partial q_i \partial p_j} & - \frac{\partial^2
      H}{\partial q_i \partial q_j}\\
      \frac{\partial^2 H}{\partial p_i \partial p_j} & - \frac{\partial^2
      H}{\partial p_i \partial q_j}
    \end{array}\right) + O(dt^2) = 1 + O(dt^2)$.
}, which makes the computation of log-likelihood cheaper than for a general ODE flow. We show in \Cref{fig.results} that this does not necessarily reduce the expressivity of  $p(q_n)$ even if $p(s_n)$ may be restricted; (ii) General ODE flows are only invertible in the limit $dt \rightarrow 0$, whereas for some Hamiltonians, we can use simplectic integrators (such as Leap-Frog, \cite{neal2011mcmc}) that are both invertible and volume-preserving for any $dt > 0$.
The resulting density $p_{\theta}(s_n)$ is given by $\ln p(s_n) = \ln \pi(s_0) = \ln \pi ( T_{H_1}^{-dt} \circ \ldots \circ T_{H_n}^{-dt}(s_n) ) + O(dt^2).$

The structure $s = (q, p)$ on the state-space imposed by the Hamiltonian dynamics can be constraining from the point of view of density estimation, since this would require an artificial split of the data vectors in two disjoint sets. We could consider different mechanisms to address this: (i) Exploit the splits and use alternating masks similar to \cite{dinh2016density}; (ii) Treat the momentum variables $p$ as latent variables. The latter is the same interpretation as in HMC \cite{neal2011mcmc, salimans2015markov, levy2017generalizing}. It is also more elegant as it does not require an artificial split of the data. This results in a density $p(q_n)$ of the form $p(q_n) = \int d p_n p(q_n, p_n) = \int d p_n \pi ( T_{H_1}^{-dt} \circ \ldots \circ T_{H_n}^{-dt}(q_n, p_n)).$ This integral is intractable, but the model can still be trained via variational methods where we introduce a variational density $h_{\phi}(p_n | q_n)$ with parameters $\phi$ and optimise the ELBO,
\begin{align}
     \text{ELBO}(q_n) &= \mathbb{E}_{h_{\phi}(p_n | q_n)}[ \ln \pi ( T_{H_1}^{-dt} \circ \ldots \circ T_{H_n}^{-dt}(q_n, p_n)) - \ln h_{\phi}(p_n | q_n) ] \leq \ln p(q_n), \label{eq.elbo}
\end{align}
instead. Note that, in contrast to VAEs (\cite{kingma2013auto, rezende2014stochastic}), the ELBO \eqref{eq.elbo} is not explicitly in the form of a reconstruction error term plus a KL term.

\paragraph{Equivariant Hamiltonian Flows and Symmetries}
\label{sec.symmetry.groups}
Let's say we want to learn a density $p_{\theta}(x)$ and we also want it to be invariant with respect to a set of transformations $x' = T_{\omega}(x)$ induced by the elements $\omega$ of a symmetry group $G$. That is, $p(x')=p(T_{\omega}(x))=p(x), \forall \omega \in G$. In the following we will assume that $G$ is a connected Lie group\footnote{See \citep{kirillov2008introduction} for an introduction to Lie groups and their infinitesimal counterparts, the Lie algebras. Readers unfamiliar with these concepts should think of Lie groups as abstract equivalent of groups of matrices. For example, the group of orientation preserving rotations $\rm{SO}(n)$ is a Lie group, and its infinitesimal counterpart (its Lie algebra) is the space of anti-symmetric matrices.}. This will allow us to introduce infinitesimal transformations, and being invariant with respect to transformations in $G$ will be equivalent to being invariant with respect to those infinitesimal transformations.

The Hamiltonian formalism provides a natural language to manipulate symmetries and, in particular, it provides a connection between invariance of the Hamiltonian function to the equivariance of the Hamiltonian flow. We can parametrize the action of an infinitesimal symmetry transformation on the state-space $s=(q, p)$ by a scalar function $g(s)$ via the Poisson bracket as $s' = s + \epsilon \{s, g\}$, where $\epsilon > 0,~\epsilon << 1$. This can also be represented as $s' = T_g^{\epsilon}(s)$.
In this context, a transformation $T_g^{\epsilon}$ is defined to be an infinitesimal symmetry if $T_g^{\epsilon} \circ T_H^{dt} = T_H^{dt} \circ T_g^{\epsilon} + O(\epsilon dt^2 + \epsilon^2 dt)$. That is, we can interchange the composition order of the Hamiltonian flow by the symmetry transformation. This is illustrated in \Cref{fig.equivariance}(left). This means that all intermediate states $s_{1, \ldots, n}$ of the Hamiltonian flow will also transform in the same manner as the starting state, that is, they will form an {\it equivariant} or {\it covariant} representation of the starting state.
When $T_g^{\epsilon}$ is a symmetry, the function $g$ is called a {\it symmetry generator}.
By Noether's first Theorem \cite{noether1971invariant}, if $T_g^{\epsilon}$ is a symmetry then the generator $g$ commutes with the Hamiltonian function, that is $\{H,~g\} = 0$ \footnote{This can be easily proven by observing that  $T_g^{\epsilon} \circ T_H^{dt} - T_H^{dt} \circ T_g^{\epsilon} =  \epsilon dt \{g, H\}  + O(\epsilon dt^2 + \epsilon^2 dt)$.}.
This implies that the numerical value of the function $g$ will {\it remain constant throughout the Hamiltonian flow}, for this reason the symmetry generators $g$ are also referred to as {\it conserved charges} in physics. Based on this reasoning, we say that a Hamiltonian flow is equivariant if $\{H,~g\} = 0$.
The set of all symmetry generators is closed under the Poisson bracket. That is, if two functions $f$ and $g$ are symmetry generators then $\{f, g\}$ is a symmetry generator. This means that the set of all generators forms a {\it Lie algebra}. 

If we can express our domain knowledge about the target density as: (i) a set of symmetry generators $g_k$, $k=1,\ldots,N_g$ and (ii) an initial simple invariant density $\pi$, we can construct a new density $p$ via Hamiltonian flow that is also invariant by learning a Hamiltonian such that $\{g_k, H\}=0$. This is formalised in \Cref{lemma.main}.

\begin{lemma}\label{lemma.main} Given a Hamiltonian function $H: \mathbb{R}^{2d} \rightarrow \mathbb{R}$, a set of symmetry generators $g_k: \mathbb{R}^{2d} \rightarrow \mathbb{R}$ and a base density $\pi: \mathbb{R}^{2d} \rightarrow \mathbb{R}^+$, the density $p: \mathbb{R}^{2d} \rightarrow \mathbb{R}^+$ induced by the Hamiltonian flow $T_H^{dt}$ will be invariant with respect to the generators $g_k$ if $\{g_k,~H\}=\{g_k,~\pi\}=0,~\forall k$.
\end{lemma}

\begin{proof}
\vskip -0.1in
From Noether's theorem, $\{g_k, H\}=0$ implies the equivariance of the flow $T_g^{\epsilon} \circ T_H^{dt} = T_H^{dt} \circ T_g^{\epsilon} + O(\epsilon dt^2 + \epsilon^2 dt)$. It remains to prove that $\pi(T_g^{\epsilon}(s))=\pi(s) + O(\epsilon^2)$. Expanding to first order in $\epsilon$, we have $\pi(T_g^{\epsilon}(s)) = \pi(s) + \epsilon \nabla \pi^T\{s, g_k(s)\} + O(\epsilon^2) = \pi(s) + \epsilon \{\pi, g_k(s)\} + O(\epsilon^2) = \pi(s) + O(\epsilon^2)$.
\end{proof} 

\vskip -0.1in
We can enforce $\{g_k,~H\}=0$ via constrained optimisation, where we perform a min-max optimisation of the Lagrangian $L(\theta, \phi, \lambda)$ instead,
\begin{align}
    (\theta^{\star}, \phi^{\star}, \lambda^{\star}) &= \min_{\theta, \phi}\  \max_{\lambda \geq 0}\ L(\theta, \phi, \lambda) \nonumber; \quad 
    L(\theta, \phi, \lambda) &= -\sum_{x \in \text{Data}} \text{ELBO}(x) + \sum_k \lambda_k \mathbb{E}_{\pi} [\{g_k(s), H(s)\}^2 - \kappa],
    \label{eq.elbo.constrained.generators}
\end{align}
where $\lambda_k$ is a Lagrange multiplier for the $k$th generator and $\kappa$ controls the precision of the constraint.

In lemma \ref{lemma.main}, we found how to get the density over the full state $s$ to be invariant under the symmetries. What we want though, is to have the marginal distribution over $q$ to be invariant. This will happen when the Hamiltonian $H(q, p)$ is factored as a "kinetic" $K(p)$, and "potential" $U(q)$ energy terms, $H(q, p) = K(p) + U(q)$.\footnote{This condition puts a constraint on the Hamiltonian to ensure that any learnt generator will induce an action on $q$ such that the marginal distribution is invariant under that action. Alternatively, we could have put a constraint on the type of generators (namely only consider those $g$ such that $\frac{\partial^2g}{\partial p_i\partial p_j}=0,~\forall i,j$) such that any learnt Hamiltonian will do. This condition is also satisfied in our experiments.}


{\bf Connections to disentangling: } 
Given a decomposition of a continuous group into a direct product of subgroups, the actions of the generators of these subgroups can be modelled by independent generators $T^{\epsilon}_w$ and $T^{\epsilon}_g$ such that $\{T^{\epsilon}_w, T^{\epsilon}_g\}=0$. 
By definition in \cite{higgins2018towards}, these will act on independent subspaces of a disentangled $q$. Since equivariant flows preserve the action of the generators, they will preserve the disentanglement of the representation. 
\vspace{-0.15in}

\section{Experiments}

For all experiments the Hamitonian was of the form $H(q, p) = K(p) + U(q)$. The kinetic energy term $K$ and the potential energy term $U$ are soft-plus MLPs\footnote{Due to the Poisson bracket, optimisation of Hamiltonian flows involves second-order derivatives of the MLPs used for Hamiltonians and generators, so relu non-linearities are not suitable.} with layer-sizes $[d, 128, 128, 1]$ where $d$ is the dimension of the data. The encoder network was parametrized as $h(p|q) = N(p; \mu(q), \sigma(q))$, where $\mu$ and $\sigma$ are relu MLPs with size $[d, 128, 128, d]$. The Hamiltonian flow $T_H^{dt}$, was approximated using a Leap-Frog integrator \cite{neal2011mcmc} since it preserves volume and is invertible for any $dt$. We found that only two Leap-Frog steps where sufficient for all our examples. Parameters were optimised using Adam \cite{kingma2014adam} (learning rate 3e-4) and Lagrange multipliers were optimised using the same method as \cite{rezendegeneralized}. The initial density $\pi$ was chosen separately for each dataset.

{\bf Domain knowledge via generators: }In this experiment, we test the main idea of this paper where we want to learn a target density which is symmetric with respect to the SO(2) group of 2D rotations. In \Cref{fig.results}A we show the target density. The Lie-algebra of the SO(2) group has a single generator which we express as $g(q, p) = q_1 p_2 - q_2 p_1$. For this experiment, the base density is a spherical normal $\pi(s)=N(s;0, \mathbb{I})$ since it has $SO(2)$ symmetry.
Our results demonstrate that the inclusion of the symmetry constraint increases the data-efficiency of the model in the infinite data regime  and, most importantly, substantially reduces overfitting in the finite data-regime.

\begin{figure*}[t]
\vskip -0.3in
\begin{center}
\centerline{\includegraphics[width=\linewidth]{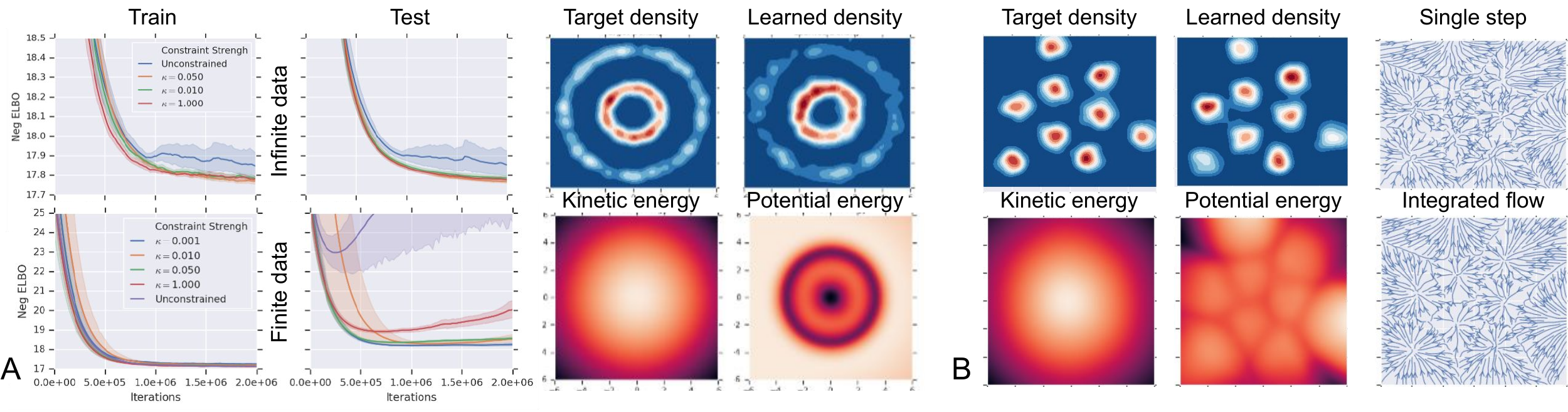}}
\caption{Learning with Hamiltonians $H(q, p) = K(p) + U(q)$.
\textbf{A}: Learning a density with known SO(2) symmetry. First two columns: training and test log-likelihoods of models with different constraint precision $\kappa$ in the regime of infinite and finite data. Adding the symmetry constraint increases both train and test data efficiency. In the finite data regime it prevents over-fitting. Second two columns: KDE estimators of the target and learnt densities; learned kinetic energy $K(p)$ and potential energy $U(q)$. 
\textbf{B}: Multimodal density learning. First two columns: KDE estimators of the target and learnt densities; learned kinetic energy $K(p)$ and potential energy $U(q)$. Last column: single Leap-Frog step and integrated flow. The potential energy learned multiple attractors, also clearly visible in the integrated flow plot. The basins of attraction are centred at the modes of the data. 
\label{fig.results}}
\end{center}
\vskip -0.35in
\end{figure*}

{\bf Multimodal density Learning: } In this experiment we to assess the expressivity of Hamiltonian flows. We demonstrate in \Cref{fig.results}B that it can transform a soft-uniform\footnote{The soft-uniform density $\pi(s; \sigma, \beta) \propto f(\beta (s + \sigma \frac{1}{2})) f(-\beta (s - \sigma \frac{1}{2}))$, where $f$ is the sigmoid function was chosen to make it easier to visualise the learned attractors. The experiment also work if we start from other densities such as a Normal.} $\pi(s; \sigma, \beta)$ into a new density with arbitrary number of modes. Furthermore, the resulting model is nicely interpretable: The learned potential energy $U(q)$ learned to have several local minima, one for each mode of the data. As a consequence, the trajectory of the initial samples through the flow has attractors at the modes of the data.
\section{Conclusions}
We have demonstrated an effective method for learning invariant probability densities using Hamiltonian flows. The proposed method opens the doors to many potential applications of ML in physics, which we plan to explore in future work.

\clearpage
\bibliographystyle{unsrt}
\bibliography{bibliography}

\end{document}